\documentclass[twoside]{article}

\usepackage[preprint]{aistats2027}
\usepackage[utf8]{inputenc}
\usepackage[T1]{fontenc}
\usepackage{amsmath, amssymb, mathtools}
\usepackage{amsthm}
\newtheorem{proposition}{Proposition}
\usepackage{bm}
\usepackage{dsfont}
\usepackage{float}
\usepackage{graphicx}
\usepackage{booktabs}
\usepackage{enumitem}
\usepackage{microtype}
\usepackage{subcaption}
\usepackage{nicefrac}
\usepackage{xcolor}
\usepackage[round]{natbib}

\usepackage{xurl}
\usepackage[hidelinks]{hyperref}
\usepackage{todonotes}
\usepackage{siunitx}

\renewcommand{\u}{\mathbf{u}}
\let\oldv\v
\renewcommand{\v}{\mathbf{v}}
\newcommand{\y}{\mathbf{y}}

\newcommand{\x}{\mathbf{x}}

\newcommand{\z}{\mathbf{z}}
\newcommand{\I}{\mathbf{I}}

\newcommand{\yobs}{\mathbf{y}_{\mathrm{obs}}}

\let\oldc\c
\renewcommand{\c}{\mathbf{c}}

\newcommand{\thetab}{\bm{\theta}}

\newcommand{\epsilonb}{\bm{\epsilon}}

\newcommand{\psib}{\bm{\psi}}

\newcommand{\0}{\bm{0}}

\newcommand{\R}{\mathbb{R}}

\newcommand{\postobs}{p(\thetab \mid \yobs)}

\newcommand{\normal}{\mathcal{N}}
\newcommand{\uniform}{\mathcal{U}}
\newcommand{\given}{\mid}
\newcommand{\E}{\mathbb{E}}
\newcommand{\Var}{\operatorname{Var}}

\newcommand{\loss}{\mathcal{L}}

\newcommand{\diff}{\,\mathrm{d}}

\newcommand{\flowmap}{\Phi_{\psib}}
\newcommand{\refsubfigure}[1]{#1}

\begin{document}

%

%


\twocolumn[

\aistatstitle{Conditional Mean Flows: Few-Step Densities for Simulation-Based Inference}

\aistatsauthor{ Jonas Arruda \And Jan Hasenauer \And Stefan T.~Radev }

\aistatsaddress{ University of Bonn \And  University of Bonn \And Rensselaer Polytechnic Institute } ]

\begin{abstract}
Flow matching has emerged as a powerful generative family for simulation-based inference (SBI). 
However, sampling and density evaluation require numerical integration at inference time. 
This makes posterior sampling in neural posterior estimation costly. More importantly, it can make repeated likelihood evaluation prohibitively expensive and impractical for Markov chain Monte Carlo (MCMC) sampling, which may require thousands of repeated density evaluations.
Here, we develop conditional mean flows for posterior and likelihood approximation, combining few-step sampling with tractable density evaluation.
We show that the mean flow objective identifies conditional flow maps and that independent masking enables direct estimation of marginal likelihoods. We evaluate both posterior and likelihood estimation across several tasks and show that conditional mean flows retain the flexibility of flow-based density estimators while significantly reducing inference time. In particular, single-step conditional mean flow density evaluation yields estimation error comparable to 200-step flow matching while enabling about 200-fold faster downstream MCMC sampling.
\end{abstract}

\section{INTRODUCTION}
\label{sec:introduction}

Mechanistic simulators are used to describe processes ranging from cognitive science \citep{ratcliff2016diffusion} to disease modeling \citep{ye2025integrating}. Yet, while practically important, their likelihoods are often unavailable in closed form.
Simulation-based inference \citep[SBI;][]{sisson2018handbook, cranmer2020frontier, deistler2025simulation-based, arruda2025diffusion} addresses this challenge by using simulated parameter-observation pairs $(\thetab, \y)$ to infer unknown parameters, and amortized methods reuse the trained network across datasets \citep{buerkner2023some}.

Two popular forms of SBI are neural posterior estimation (NPE) and neural likelihood estimation (NLE). NPE learns a posterior model $q_{\psib}(\thetab\given\y)$ which can sample from the approximate posterior directly.
NLE, instead, learns a likelihood model $q_{\psib}(\y\given\thetab)$ which requires Markov chain Monte Carlo (MCMC) to sample from the posterior.
NPE amortizes over datasets but is usually tied to the prior and data format used in training, whereas NLE needs MCMC for every dataset, but allows the prior to be changed after training and composes exchangeable likelihoods across any number of i.i.d.\ observations \citep{cranmer2020frontier}.

Many SBI methods demand neural estimators that are both flexible and inexpensive to evaluate.
In NLE, MCMC requires thousands of likelihood evaluations per posterior, making each evaluation a practical bottleneck.
Further, hierarchical approaches \citep{arruda2024amortized} or self-consistency objectives that improve NPE \citep{schmitt2024leveraging, mishra2026robust} and enable model comparison \citep{kucharsky2026amortized} rely on density evaluation during training or inference.

So far, normalizing flows have been a standard choice in these applications because they evaluate densities in a single pass \citep{papamakarios2021normalizing, dinh2017density, durkan2019neural}.
However, their transformations must remain invertible with tractable Jacobians, which restricts how inductive biases can be built into the architecture \citep{chen2025conditional}.
Autoregressive flows based on causal transformers (e.g., TARFlow and STARFlow) relax this restriction \citep{zhai2025normalizing, gu2025starflow}, but impose a fixed variable ordering and require sequential inversion for sampling. Moreover, tractable joint densities do not generally imply tractable marginals, which are required with missing observations \citep{yalavarthi2024marginalization}.

Flow matching and diffusion models instead learn a velocity or score field with an unconstrained neural network \citep{lipman2023flow, liu2022flow, wildberger2023flow, sharrock2024sequential, arruda2025diffusion}, and diffusion transformers can flexibly represent arbitrary conditionals and marginals \citep{gloeckler2024all-in-one}.
Yet, their densities require solving an ordinary differential equation (ODE), which can take hundreds of network evaluations and makes repeated likelihood calls during training or within MCMC costly.
Few-step generative models reduce the sampling cost, but consistency models, for instance, provide no tractable density \citep{song2023consistency, schmitt2024consistency}.
Thus, \textit{current SBI approaches either restrict the network architecture or make every density evaluation expensive}.

To overcome this limitation, we present conditional mean flows for SBI, which learn the average velocity between two times, so that one network evaluation defines a finite-time map based on mean flows \citep{geng2025mean, geng2026improved}.
We first evaluate direct posterior sampling and density estimation, then focus on learned likelihoods for compositional inference, missing observations, and gradient-based and delayed acceptance MCMC.
Concretely, we develop and showcase
\begin{enumerate}
    \item a conditional mean flow formulation for posterior and likelihood approximation, with exact stepwise Jacobian determinants, and a proof that the mean flow objective suffices to identify conditional flow maps (\autoref{sec:flow-densities}, Proposition~\ref{prop:conditional});
    \item marginal density estimation for missing targets with a masked transformer backbone, and a proof that the masked objective targets the correct marginal likelihood (\autoref{sec:backbone}, Proposition~\ref{prop:marginal});
    \item and a benchmark against normalizing flows and regular flow matching on four tasks covering both NLE and NPE, including a single-trial likelihood comparison in which one mean flow step matches the error of 200-step flow matching with two orders of magnitude faster sampling (\autoref{sec:experiments});
\end{enumerate}

\begin{figure*}[t]
    \centering
    \includegraphics[width=\textwidth]{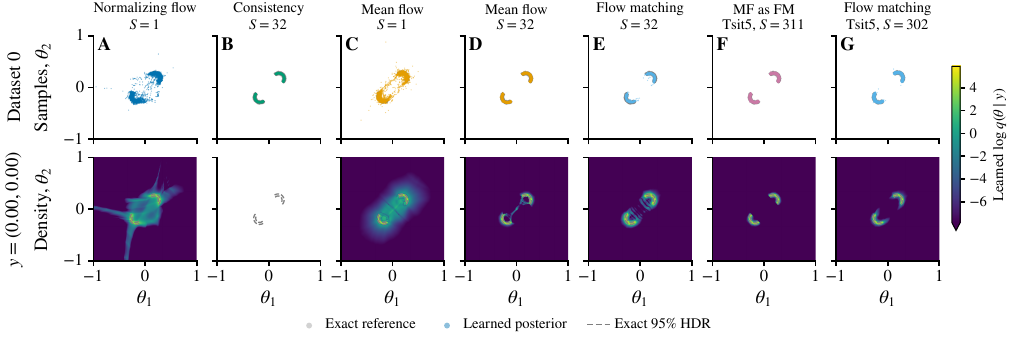}
    \caption{\emph{Learned two-moons posteriors as samplers and as densities.}
    Top row: 5{,}000 draws per method for the canonical observation $\y = (0,0)$; bottom row: the learned log-density on a $512{\times}512$ grid with the exact 95\% highest-density region dashed.
    (\textbf{A})~A normalizing flow; 
    (\textbf{B})~a consistency model with $S{=}32$ steps, which has no tractable density; 
    (\textbf{C})~a mean flow with $S{=}1$.
    (\textbf{D})~The same mean flow with $S{=}32$ and (\textbf{E})~a separately trained flow matching model integrated with 32 Euler steps.
    (\textbf{F})~The velocity head of the mean flow and (\textbf{G})~the flow matching model, both integrated with the adaptive Tsit5 ODE solver.}
    \label{fig:overview}
\end{figure*}

\section{BACKGROUND AND RELATED WORK}
\label{sec:background}

\paragraph{Simulation-based inference.}
For conditionally independent observations $\y_n$, simulation-based inference targets $\postobs \propto p(\thetab)\, \prod_{n=1}^{N} p(\y_n \given \thetab)$ when the likelihood is unavailable.
Classical approximate Bayesian computation compares simulated and observed data directly, but becomes less effective as dimension increases \citep{sisson2018handbook}.
Neural posterior estimation (NPE) fits $q_{\psib}(\thetab\given\y)$ to simulated pairs and draws approximate posterior samples without running MCMC for each dataset.
Neural likelihood estimation (NLE), instead, fits a conditional density to prior-predictive pairs and uses it within MCMC \citep[for an overview, see][]{cranmer2020frontier}.
A per-observation estimator can be reused across dataset sizes, making single-trial likelihoods useful for models with exchangeable trials \citep{boelts2022flexible}, and spline coupling flows are the standard tractable-density baseline~\citep{durkan2019neural}.

\paragraph{Flow matching and diffusion-based inference.}
Flow matching learns a velocity field $\v(\z_t, t)$ that transports a Gaussian distribution to the target along a prescribed interpolation \citep{lipman2023flow, liu2022flow}.
It has been applied to both NPE and NLE, alongside score-based generative models \citep[for a review, see][]{arruda2025diffusion}.
Sampling requires ODE integration, and density evaluation additionally integrates the velocity divergence through the continuous change-of-variables formula \citep{chen2018neural}, so repeated integration can dominate the cost of NLE.

\paragraph{Few-step generative models.}
Consistency models learn transitions from intermediate noisy states toward data, enabling generation in one or a few steps \citep{song2023consistency}.
They also support fast posterior sampling in SBI \citep{schmitt2024consistency, arruda2025diffusion}, although the standard formulation does not provide a tractable posterior density.
Flow maps learn finite-time ODE transitions between arbitrary time pairs, and mean flows parameterize these maps through average velocities and can be trained without distillation \citep{geng2025mean, geng2026improved}.
\citet{rehman2026falcon} combined mean flows with a cycle loss that encourages invertibility to obtain few-step densities for Boltzmann generators.

\paragraph{MCMC with learned and multi-fidelity likelihoods.}
The no-U-turn sampler \citep[NUTS;][]{hoffman2014no-u-turn, betancourt2017conceptual} uses density and gradient evaluations to construct proposals, and probabilistic programming frameworks allow learned likelihoods to enter these calculations as model components \citep{carpenter2017stan}.
Delayed acceptance MCMC uses a cheap approximation to screen proposals while leaving the fine target invariant \citep{christen2005markov}.
Moreover, its multilevel variant organizes several approximations \citep{lykkegaard2023multilevel}, and related approaches in SBI use simulators at different resolutions \citep{prescott2021multifidelity, tatsuoka2026bifidelity}.

\section{METHOD}
\label{sec:method}

\paragraph{A note on notation.}
We write $\z$ for the density target and $\x$ for the conditioning variables.
In \emph{posterior estimation}, the network learns $q_{\psib}(\thetab\given\y)$ with parameters $\z\equiv\thetab$ as the target and data $\x\equiv\y$ as the condition.
In \emph{likelihood estimation}, it learns $q_{\psib}(\y\given\thetab)$ with $\z\equiv\y$ and $\x\equiv\thetab$.
The same formulation applies to either direction by choosing the target and condition accordingly.

\subsection{Flow-Based Density Evaluation}
\label{sec:flow-densities}

\paragraph{Flow matching vs.~normalizing flows.} We consider flows that transport the target distribution $p(\z \given \x)$ to a spherical Gaussian base distribution.
Time runs from data to noise: $t=0$ is the clean target and $t=1$ is pure noise, along the flow matching interpolation $\z_t = (1-t)\,\z_0 + t\,\epsilonb$ with $\epsilonb \sim \normal(\0, \I)$.
Flow matching for SBI learns the marginal velocity $\v_{\psib}(\z_t, t, \x)$ by regressing the conditional velocity $\epsilonb - \z_0$. For sampling, it draws the initial condition $\z_1\sim\normal(\0, \I)$ and then integrates $\mathrm{d}\z_t = \v_{\psib}\diff t$ from $t{=}1$ to $t{=}0$ \citep{wildberger2023flow}.
The density follows from the continuous change of variables,
\begin{equation}
    \log q_{\psib}(\z_0 \given \x) = \log \normal(\z_1; \0, \I) + \int_0^1 \nabla\!\cdot \v_{\psib}(\z_t, t, \x) \diff t,
    \label{eq:cnf}
\end{equation}
which can be integrated with any ODE solver \citep{arruda2025diffusion}.
Normalizing flows instead learn an invertible map $f_{\psib}(\cdot\,; \x)$ from the target to the base distribution and evaluate
\begin{equation}
    \begin{split}
    \log q_{\psib}(\z \given \x) ={}& \log \normal\big(f_{\psib}(\z; \x); \0, \I\big) \\
    &+ \log \big|\det \nabla_{\z} f_{\psib}(\z; \x)\big|
    \end{split}
    \label{eq:nf}
\end{equation}
in a single pass \citep{papamakarios2021normalizing}.
To keep the determinant tractable, $f_{\psib}$ is composed of invertible layers with triangular Jacobians \citep{dinh2017density, durkan2019neural}, which avoids the integral in \eqref{eq:cnf} but restricts the neural network architecture~(\autoref{fig:overview}).

\paragraph{Average velocity and the flow map.}
Extending \citet{geng2025mean} to conditional models, we train a network to predict the \emph{average velocity}
\begin{equation}
    \u_{\psib}(\z, s, t, \x) \approx \frac{1}{t-s} \int_s^t \v(\z_\tau, \tau, \x)\diff\tau ,
    \label{eq:average-velocity}
\end{equation}
the mean of the instantaneous field along the ODE trajectory started at $\z_s = \z$ conditioned on $\x$.
Since $\x$ is constant along the trajectory, the conditional velocity defines one ODE per condition.
One network evaluation then applies the finite-time \emph{flow map}
\begin{equation}
    \flowmap(\z; s, t, \x) = \z + (t - s)\,\u_{\psib}(\z, s, t, \x),
    \label{eq:flow-map}
\end{equation}
with the boundary case $\u_{\psib}(\z, t, t, \x)$ corresponding to the instantaneous velocity $\v$.

\paragraph{Training with a self-consistent tangent and a cycle loss.}
Differentiating $(t-s)\,\u = \int_s^t \v\diff\tau$ along the trajectory gives the MeanFlow identity \citep{geng2025mean},
\begin{equation}
    \v(\z_s, s, \x) = \u_{\psib}(\z_s, s, t, \x) - (t - s)\, \tfrac{\mathrm{d}}{\mathrm{d} s}\, \u_{\psib}(\z_s, s, t, \x),
    \label{eq:meanflow-identity}
\end{equation}
where $\nicefrac{\mathrm{d}}{\mathrm{d} s} = \v \cdot \nabla_{\z} + \partial_s$ is the total derivative along $(\dot\z_s, \dot s, \dot t) = (\v, 1, 0)$ at fixed $\x$.
Following \citet{geng2026improved}, we use the right-hand side as a prediction of the instantaneous velocity,
\begin{equation}
    \mathbf{V}_{\psib} = \u_{\psib}(\z_s, s, t, \x) - (t - s)\, \operatorname{sg}\!\big[\tfrac{\mathrm{d}}{\mathrm{d} s}\u_{\psib}(\z_s, s, t, \x)\big],
    \label{eq:compound-velocity}
\end{equation}
where $\operatorname{sg}[\cdot]$ denotes the stop-gradient and the total derivative is computed by one Jacobian-vector product with tangent $(\operatorname{sg}[\hat{\v}_{\psib}], 1, 0)$.
Since the marginal velocity $\v$ is unknown, the tangent $\hat{\v}_{\psib}(\z_s, s, \x)$ is predicted by an auxiliary velocity head, which shares all but the last two layers with $\u_{\psib}$.
With $\z_s = (1-s)\,\z_0 + s\,\epsilonb$, both outputs regress onto the conditional velocity,
\begin{equation}
    \loss_{\mathrm{MF}}(\psib) = \E\Big[ w\,\big\|\mathbf{V}_{\psib} - (\epsilonb - \z_0)\big\|^2 + \tfrac{1}{d}\big\|\hat{\v}_{\psib} - (\epsilonb - \z_0)\big\|^2 \Big],
    \label{eq:mf-loss}
\end{equation}
with the adaptive weight $w = \operatorname{sg}\big[(\|\mathbf{V}_{\psib} - (\epsilonb - \z_0)\|^2 + c)^{-p}\big]$ of \citet{geng2025mean} and $d$ the dimension of $\z$ (all hyperparameters are defined in \autoref{app:training}).
The expectation is taken over simulations $(\z_0, \x)$, noise $\epsilonb \sim \normal(\0, \I)$, and independent times $s, t$ drawn from a power-law distribution.
Hence, times are sampled more often closer to the target, similar to \citet{wildberger2023flow}.
For some of the pairs we set $t{=}s$, where $\mathbf{V}_{\psib} = \u_{\psib}(\z_s, s, s, \x)$ and \eqref{eq:mf-loss} reduces to flow matching.
Besides stabilizing training, the velocity head allows us to use the trained mean flow as a flow matching model.
To encourage invertibility of the flow map, we add the cycle loss of \citet{rehman2026falcon}, which penalizes the discrepancy between a state and its image after mapping $s \to t$ and back,
\begin{equation}
    \loss_{\mathrm{cyc}}(\psib) = \E\Big[\tfrac{1}{d}\big\|\z_s - \flowmap\big(\flowmap(\z_s;\, s, t, \x);\, t, s, \x\big)\big\|^2\Big],
    \label{eq:cycle-loss}
\end{equation}
and minimize $\loss_{\mathrm{MF}} + \lambda\,\loss_{\mathrm{cyc}}$ with the cycle weight $\lambda > 0$.
We further replace some $\x$ in \eqref{eq:cycle-loss} by the condition of another simulation, to encourage $\flowmap$ to be invertible for every condition and to make the map robust to any MCMC proposal.

\paragraph{Exact densities from composed steps.}
We fix a schedule $0 = t_0 < t_1 < \dots < t_S = 1$ and let $G_i(\z) = \flowmap(\z;\, t_{i-1}, t_i, \x)$ be the step maps toward noise.
With $\z^{(0)} = \z$ and $\z^{(i)} = G_i(\z^{(i-1)})$, the change-of-variables expression is
\begin{equation}
\begin{aligned}
    \log q_{\psib}^{(S)}(\z \given \x)
    &= \log \normal\big(\z^{(S)}; \0, \I\big) \\
    &\quad + \sum_{i=1}^{S} \log \big|\det \nabla_{\z} G_i\big(\z^{(i-1)}\big)\big|,
\end{aligned}
\label{eq:composed-density}
\end{equation}
where each Jacobian can be computed via autodiff.
If the composed map $T = G_S \circ \dots \circ G_1$ is a diffeomorphism of $\R^d$, the chain rule gives $\det \nabla_{\z} T = \prod_{i} \det \nabla_{\z} G_i$, so \eqref{eq:composed-density} is the pullback of the base density through $T$.
We place the nodes by an EDM schedule \citep{karras2022elucidating},
$t_i = (\sigma_i - \sigma_{\min})/(1 - \sigma_{\min})$ with $\sigma_i = \big(\sigma_{\min}^{1/\rho} + \tfrac{i}{S}(1 - \sigma_{\min}^{1/\rho})\big)^{\rho}$,
which concentrates the steps near the data for $\rho > 1$ and has been shown to outperform other schedules for few-step density evaluation~\citep{rehman2026falcon}.

\paragraph{Conditioning.}
Although the mean flow only receives $\x$ as an additional input and regresses the same target $\epsilonb - \z_0$ as an unconditional model, we show that it learns the correct conditional flow map.
Similar to conditional diffusion models \citep{vincent2011connection, batzolis2021conditional, arruda2025diffusion}, the condition only restricts which targets $\z_0$ contribute to the regression.

\begin{proposition}[Conditional mean flows]
\label{prop:conditional}
Under the assumptions of Appendix~\ref{app:conditional}, every population minimizer of the unweighted loss $\loss_{\mathrm{MF}}$ satisfies $\u_{\psib}=\u^*$ for $p(\x)$-almost every $\x$, where $\u^*$ is the average velocity \eqref{eq:average-velocity} of $\v^*(\z,s,\x)=\E[\epsilonb-\z_0\given\z_s=\z,\x]$.
Hence, \eqref{eq:composed-density} equals $\log p(\z \given \x)$ for every schedule.
\end{proposition}
The proof is given in Appendix~\ref{app:conditional}.

\subsection{Sampling and Posterior Inference}
\label{sec:sampling}

For sampling, we draw $\z^{(S)} \sim \normal(\0, \I)$ and apply the flow map backward along the same schedule, $\z^{(i-1)} = \flowmap(\z^{(i)};\, t_i, t_{i-1}, \x)$ for $i = S, \dots, 1$, so that $S$ steps yield a sample $\z^{(0)}$.
In the posterior direction, these are parameter draws conditional on the observed data; in the likelihood direction, they are observation draws conditional on the parameters.
These samples follow the density in \eqref{eq:composed-density} only if the backward maps invert the forward maps $G_i$, which the cycle loss encourages.

In the likelihood direction, the joint log-posterior of $N$ i.i.d.\ observations is proportional to $\log p(\thetab) + \sum_{n=1}^N \log q_{\psib}^{(S)}(\y_n \given \thetab)$, which is differentiable in $\thetab$ for any fixed $S$ with nonsingular step Jacobians.
We expose it to \texttt{PyMC} as a distribution  \citep{salvatier2016probabilistic} and sample with \texttt{nutpie}'s NUTS on the \texttt{JAX} backend \citep{seyboldt2026preconditioning}.

Additionally, we can construct a two-level sampler as a trained mean flow provides a coarse likelihood $q_{\psib}^{S_c}$ with few steps and a fine likelihood $q_{\psib}^{\mathrm{f}}$ by integrating its velocity head with an accurate ODE solver \eqref{eq:cnf}.
This defines a coarse and a fine posterior,
\begin{equation}
    \begin{aligned}
    \pi_{\mathrm{c}}(\thetab) &\propto p(\thetab) \textstyle\prod_{n=1}^N q_{\psib}^{S_c}(\y_n \given \thetab), \\
    \pi_{\mathrm{f}}(\thetab) &\propto p(\thetab) \textstyle\prod_{n=1}^N q_{\psib}^{\mathrm{f}}(\y_n \given \thetab),
    \end{aligned}
    \label{eq:two-level}
\end{equation}
which we combine in two-level delayed acceptance MCMC \citep{christen2005markov, lykkegaard2023multilevel}.
From the current state $\thetab$, a subchain of $J$ random-walk Metropolis steps targeting $\pi_{\mathrm{c}}$ yields a proposal $\thetab'$, which is accepted with probability
\begin{equation}
    \alpha(\thetab, \thetab') = \min\left\{1,\, \frac{\pi_{\mathrm{f}}(\thetab')\, \pi_{\mathrm{c}}(\thetab)}{\pi_{\mathrm{f}}(\thetab)\, \pi_{\mathrm{c}}(\thetab')}\right\}.
    \label{eq:delayed-acceptance}
\end{equation}
Since the subchain is reversible with respect to $\pi_{\mathrm{c}}$, this correction makes $\pi_{\mathrm{f}}$ the invariant distribution \citep{lykkegaard2023multilevel}, while the fine likelihood is evaluated only once per subchain instead of at every proposal.

\subsection{Masked Transformers for Marginal Densities}
\label{sec:backbone}

Inspired by the Simformer \citep{gloeckler2024all-in-one}, we use masked attention to learn marginal densities $p(\z_{\mathcal O}\given\x)$ for subsets $\mathcal O$ of the target coordinates.
Our transformer tokenizes each coordinate of $\z$ and $\x$ and receives $(s,s-t)$ through adaptive layer normalization \citep{peebles2023scalable}.
The mask excludes omitted target tokens, so outputs depend on $\z$ only through $\z_{\mathcal O}$; the mean-flow and cycle losses are restricted to $\mathcal O$.

\begin{proposition}[Marginalization by masking]
\label{prop:marginal}
Under the assumptions of Appendix~\ref{app:marginal}, every population minimizer of the unweighted masked loss satisfies $\u_{\psib,\mathcal O}=\u^*_{\mathcal O}$ for every training mask $\mathcal O$ and $p(\x)$-almost every $\x$, where $\u^*_{\mathcal O}$ is the average velocity of $\v^*_{\mathcal O}(\z_{\mathcal O},s,\x)=\E[\epsilonb_{\mathcal O}-\z_{0,\mathcal O}\given\z_{s,\mathcal O}=\z_{\mathcal O},\x]$.
Hence, \eqref{eq:composed-density} in $|\mathcal O|$ dimensions equals $\log p(\z_{\mathcal O}\given\x)$ for every schedule.
\end{proposition}
The proof is given in Appendix~\ref{app:marginal}.
Choosing $\z=\thetab$ gives marginal posterior densities $p(\thetab_{\mathcal O}\given\y)$; choosing $\z=\y$ gives marginal likelihoods $p(\y_{\mathcal O}\given\thetab)$.
We evaluate the latter for missing observations in \autoref{sec:exp3}.

\section{EXPERIMENTS}
\label{sec:experiments}

We evaluate direct posterior approximation on two moons and likelihood approximation on the three further benchmarks, where the learned densities enter MCMC.
The benchmarks show the following:
\begin{itemize}
    \item \textbf{Two moons}: few steps give accurate samples, so sampling has only a minor speed-accuracy trade-off, and densities improve with more steps~(\autoref{sec:exp1}).
    \item \textbf{Cognitive model}: exchangeable likelihoods yield precision comparable to 200-step flow matching with about $200\times$ faster MCMC~(\autoref{sec:exp2}).
    \item \textbf{Epidemic renewal}: a masked transformer marginalizes missing observations and is more accurate than a normalizing flow~(\autoref{sec:exp3}).
    \item \textbf{g-and-k}: at a shared budget, more draws beat more steps, and delayed acceptance MCMC is as accurate as NUTS without gradients~(\autoref{sec:exp4}).
\end{itemize}
We compare samples with exact references using the classifier two-sample test \citep[C2ST;][]{lopez-paz2017revisiting}, and densities through the Hellinger distance (\autoref{app:metrics}).
Reference-free evaluations use normalized root mean squared error (NRMSE), effective sample size (ESS) and predictive checks, and the calibration for every trained estimator is assessed using tests
of accuracy with random points \citep[TARP;][]{lemos2023sampling-based} prior to any other evaluation (\autoref{tab:tarp}).
Tasks, architectures and generative models are described in \autoref{app:metrics} to~\ref{app:tasks}.

\begin{figure*}[t]
    \centering
    \includegraphics[width=\textwidth]{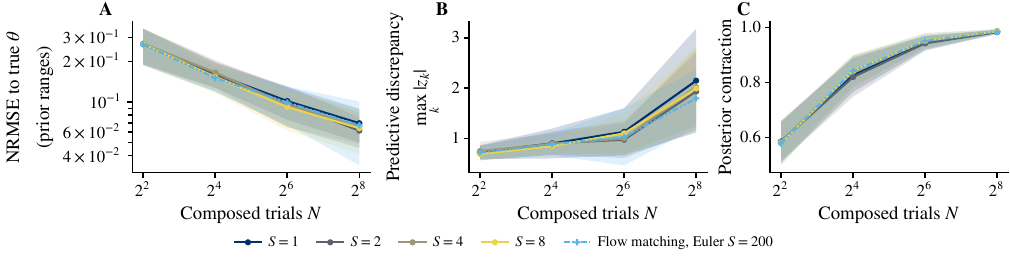}
    \caption{\emph{Composing one learned trial likelihood over $N$ trials of a cognitive model.}
    (\textbf{A})~NRMSE of the composed posterior to the generating parameter.
    (\textbf{B})~Worst posterior predictive discrepancy across choice proportion, conditional response-time quantiles and timeout rate.
    (\textbf{C})~Posterior contraction relative to the prior.
    Lines show the mean over 10 independent datasets, bands the sample standard deviation.}
    \label{fig:exp4}
\end{figure*}

\subsection{Few Steps Suffice for High Posterior Sample and Density Quality}
\label{sec:exp1}

\paragraph{Setup and baseline.}
The two moons task has a two-dimensional parameter and observation, a bimodal crescent posterior, and an exact likelihood and posterior sampler \citep{greenberg2019automatic, lueckmann2021benchmarking}.
All methods learn the \emph{posterior} direction and are sampled directly with 5{,}000 draws per dataset, which we score against 32{,}000 exact reference draws.
We compare a conditional mean flow at step counts $S \in \{1, 2, 4, 8, 16, 32\}$ with three baselines: a conditional normalizing flow, a conditional consistency model as the established few-step method \citep{song2023consistency, schmitt2024consistency}, and a separately trained conditional flow matching model integrated with $S$ Euler steps or an adaptive Tsit5 solver \citep{tsitouras2011runge--kutta}.
We integrate the velocity head of the mean flow with the same Euler or adaptive scheme, which repurposes the mean flow as a flow matching model.
For both the consistency model and the mean flow, we tuned the sampling schedule for each step count on held-out observations, which are disjoint from the evaluation datasets.

\paragraph{Sample quality.}
The conditional mean flow reaches a similar C2ST at every step count (\autoref{fig:overview} and \autoref{fig:exp1-steps}\refsubfigure{A}), so a single step suffices for sampling and already outperforms the normalizing flow ($0.523$ against $0.543$).
In contrast, the consistency model needs at least 4 steps to approach this accuracy, and both flow matching variants need at least 32 Euler steps.
The adaptive solver further improves the accuracy, but requires more than 300 solver steps.
Hence, the mean flow with $S{=}1$ draws 5{,}000 samples about 23 times faster than flow matching with 32 Euler steps and 140 times faster than with the adaptive solver (\autoref{fig:runtimes}).
It is also about 6 times faster than the normalizing flow and only reaches its sampling time at $S{=}8$.

Since the velocity head of the mean flow and the separately trained model perform similarly under both solvers, mean flow training does not seem to degrade the learned velocity field.
The consistency model is also sensitive to its discretization exponent, whose candidates span a C2ST from $0.54$ to $0.98$ at $S{=}4$, whereas the EDM schedule of the mean flow changes the C2ST by at most $0.015$~(\autoref{tab:schedule}).

\paragraph{Density quality.}
In contrast to sampling, we find that additional mean flow steps improve the density, and the Hellinger distance to the exact posterior decreases with $S$ (\autoref{fig:exp1-steps}\refsubfigure{B} and \autoref{tab:exp1-density}).
With $S{=}1$, the mean flow places mass between the two modes (\autoref{fig:overview}\refsubfigure{C}), but it still outperforms the normalizing flow.
Euler integration of either velocity field remains less accurate than the mean flow with $S{\ge}2$ up to 32 steps.
The adaptive solver is most accurate, with the velocity head of the mean flow ahead of the separately trained flow matching model.

Moreover, the mean flow density is close to normalized, whereas the Euler densities overshoot unit mass by many orders of magnitude at $S{=}1$ and remain far from normalized at $S{=}32$ (\autoref{tab:exp1-density}).
We attribute this to the Euler scheme approximating both the trajectory and the divergence integral in \eqref{eq:cnf}, whereas the mean flow computes the exact Jacobian determinant of each step.
The normalizing flow is exactly normalized, but its density shows fold artifacts and places mass outside the prior support~(\autoref{fig:overview}\refsubfigure{A}).
Flow matching models with adaptive solvers provide the most accurate densities, but more than 300 steps per evaluation make them too expensive for repeated density evaluation.

\subsection{Single-Step Likelihoods Yield Error Comparable to 200-Step Flow Matching}
\label{sec:exp2}

\paragraph{Setup and baseline.}
The cognitive model describes two-alternative decisions as a Wiener first-passage process with drift $\nu$, boundary separation $a$ and non-decision time $t_0$ \citep{ratcliff2016diffusion}, and is a standard benchmark for learned single-trial likelihoods \citep{fengler2021likelihood}.
We train one conditional mean flow on the signed response time $y = (2c-1)\,t$, with choice $c \in \{0, 1\}$ and response time $t$, and compose the learned trial likelihood into joint posteriors over $N \in \{4, 16, 64, 256\}$ trials, which we sample with NUTS.
Since a response cannot be faster than $t_0$, we truncate the prior of $t_0$ at the fastest observed response time of each run.
We compare mean flow step counts $S \in \{1, 2, 4, 8\}$ with a flow matching model integrated with 200 Euler steps.

\paragraph{Results.}
With $S{=}1$, the NRMSE and posterior contraction of the mean flow are comparable to the flow matching baseline at every $N$ (\autoref{fig:exp4}, \autoref{tab:exp4} and \autoref{fig:exp4-posteriors}), while sampling is 182 to $216\times$ faster (\autoref{fig:runtimes}\refsubfigure{D}).
Additional steps increase the cost and the number of convergence failures without clearly reducing the parameter error (\autoref{fig:exp4-steps}).
From all MCMC runs, $8.5\%$ failed to converge, and all failed runs used $S \ge 4$ (\autoref{tab:exp4}).
We attribute this to the fact that the mean flow is trained for few-step sampling and evaluation only, and hence the intermediate maps might be poorly trained and additional steps, unlike in the discretization of an ODE, need not reduce the approximation error.

However, for every method the coverage of the 95\% central interval decreases and the worst posterior predictive discrepancy increases with $N$ (\autoref{fig:exp4}\refsubfigure{B}).
Composing many approximate likelihoods accumulates their errors while the posterior contracts, which makes even the 200-step baseline overconfident.
With $S{=}1$, the additional discretization error further lowers the coverage at $N{=}256$, whereas with $S{=}2$ coverage and predictive discrepancy are comparable to the baseline.

\subsection{Masked Transformer Backbone Reduces Error under Missing Observations}
\label{sec:exp3}

\begin{figure*}[t]
    \centering
    \includegraphics[width=.95\textwidth]{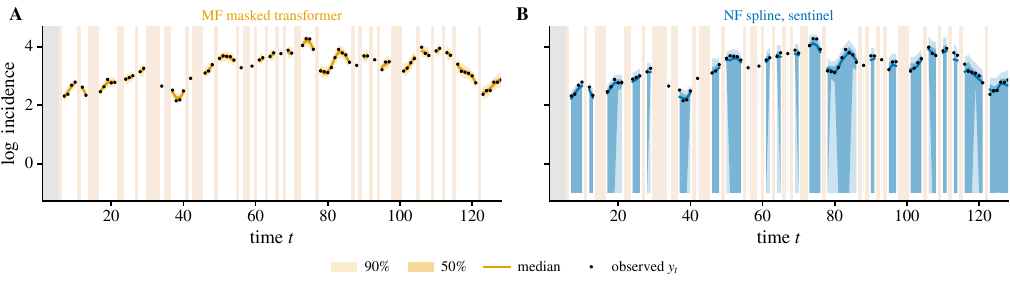}
    \caption{\emph{Posterior predictive simulations of the renewal model on one dataset.}
    Observations $\tilde{y}$ simulated from the learned likelihood at posterior draws of (\textbf{A})~our conditional mean flow with the masked transformer and (\textbf{B})~the normalizing flow with sentinel values; median with 50 and 90\% bands against the observed data (black).
    The gray strip is the seeded window, and the striped columns mark the missing time points.}
    \label{fig:exp3_prediction}
\end{figure*}

\begin{table}[t]
    \caption{\emph{Likelihoods with missing observations.}
    Mean (sample standard deviation) over 10 datasets. $\tilde{y}_t$ is simulated with the learned likelihood.}
    \label{tab:exp3}
    \centering
    \small
    \begin{tabular}{lll}
        \toprule
        & \textbf{Mean flow} & \textbf{Normalizing flow}\\
        \midrule
        NRMSE & $\mathbf{0.187}$ (0.008) & 0.199 (0.009) \\
        RMSE $\tilde{y}_t$ & $\mathbf{0.055}$ (0.007) & 0.159 (0.036) \\
        RMSE $y_t$ & $\mathbf{0.060}$ (0.013) & 0.125 (0.031) \\
        ESS & $\mathbf{4{,}234}$ (247)	& 2{,}700 (629) \\
        Leapfrog steps & $\mathbf{121}$ (4) & 480 (57) \\
        Wall time (s) & 36{,}146 (6{,}157) & $\mathbf{1{,}624}$ (285) \\
        \bottomrule
    \end{tabular}
\end{table}

\paragraph{Setup and baseline.}
Since conditional mean flows do not constrain the network architecture, we can tailor the backbone to observations with missing entries.
We consider a stochastic renewal process, a standard forward model for estimating time-varying reproduction numbers $R_t$ \citep{fraser2007estimating, cori2013new, gostic2020practical}, in which the incidence is driven by a generation-interval force (Appendix~\ref{app:task-renewal} and \autoref{fig:exp3-scenario}).
The parameter is the 128-dimensional coefficient vector $\z$ with $\log R = L\z$, where $L$ is the Cholesky factor of a squared-exponential kernel \citep{rasmussen2005gaussian}, such that the prior is $\normal(\0, \I)$.
We withhold 30\% of the observed steps completely at random and redraw this pattern for each dataset.
Under this mechanism, inference only requires the marginal likelihood $p(\y_{\mathcal{O}} \given \z)$ of the observed entries $\mathcal{O}$ \citep{little2019statistical}.
We parameterize the mean flow with a masked transformer trained with missing coordinates and a causal mask over time.
In contrast, the normalizing flow has a fixed input width, so we train and evaluate it with a value of $-1$ in the missing entries (\autoref{fig:exp3-scenario}\refsubfigure{C}).
Both models are trained on a low budget of 8{,}192 simulations, and the mean flow is evaluated with $S{=}1$.

\paragraph{Posterior accuracy and sampling cost.}
Both methods converge on every dataset without divergent transitions.
The conditional mean flow with the masked transformer has a lower NRMSE to the generating coefficients than the normalizing flow, its posterior draws reproduce the observed $y_t$ more closely under the learned likelihood and the simulator (\autoref{tab:exp3}), and its ESS per network evaluation is six times higher.
However, each evaluation of the transformer is computationally more costly, such that the normalizing flow finishes sampling $22\times$ sooner, despite requiring more leapfrog steps inside NUTS.

\paragraph{Posterior predictive.}
Simulations from the learned likelihoods at posterior draws reveal the problem of sentinel values (\autoref{fig:exp3_prediction}).
The predictive bands of the mean flow contract around the observed data at every observed time point.
In contrast, the normalizing flow frequently generates the sentinel value $-1$ also at observed time points, which is reflected in a higher predictive RMSE~(\autoref{tab:exp3}).

In summary, the flexible architecture of mean flows allows a masked transformer to marginalize missing observations.
This yields a lower posterior error than a normalizing flow with sentinel values, but at a higher cost per density evaluation.

\subsection{More Posterior Draws Are Better Than More Flow Steps}
\label{sec:exp4}

\begin{figure*}[t]
    \centering
    \includegraphics[width=\textwidth]{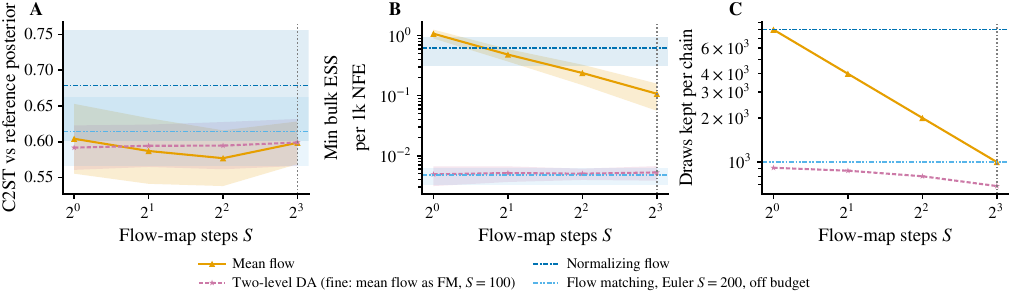}
    \caption{\emph{Budget-matched posterior accuracy on g-and-k.}
    (\textbf{A})~C2ST against the exact posterior over mean flow steps $S$, at the shared budget of Appendix~\ref{app:samplers}.
    (\textbf{B})~Minimum bulk ESS per 1{,}000 network evaluations.
    (\textbf{C})~Draws kept per chain under the shared budget.
    Lines show the mean over 10 independent datasets, bands the standard deviation; the delayed acceptance MCMC line is indexed by its coarse step count, and horizontal lines mark the normalizing flow and the flow matching model with 200 Euler steps outside the budget.}
    \label{fig:exp2-budget}
\end{figure*}

\paragraph{Setup and baseline.}
The g-and-k distribution is a standard benchmark with an intractable likelihood, defined through its quantile function \citep{rayner2002numerical, drovandi2011likelihood-free}.
We use a bivariate version with a Gaussian copula and parameters $(g, k, \rho)$, and datasets of $N{=}50$ observations (Appendix~\ref{app:task-gnk}).
The conditional mean flow learns the likelihood of a single observation, and we sample the posterior with NUTS.
To assess whether a computational budget should be spent on more mean flow steps or on more posterior draws, we fix the product of requested draws per chain and network evaluations per density (\autoref{app:samplers}).
Under this budget, we compare mean flow step counts $S \in \{1, 2, 4, 8\}$ with a normalizing flow.
In addition, we evaluate delayed acceptance MCMC, which screens proposals with the mean flow at $S_c$ steps and accepts them against a 100-step Tsit5 evaluation of the velocity head of the same network (\autoref{sec:sampling} and \autoref{app:samplers}).
A separately trained flow matching model integrated with 200 Euler steps serves as a baseline outside the budget, with 28 times more network evaluations.

\paragraph{Steps against draws.}
All mean flow settings converge on every dataset and reach a similar accuracy, with the lowest mean C2ST at $S{=}4$, although these differences are smaller than the spread across datasets (\autoref{fig:exp2-budget}\refsubfigure{A}, \autoref{tab:exp2}).
Fewer steps permit longer chains, and the minimum bulk ESS per network evaluation drops by an order of magnitude from $S{=}1$ to $S{=}8$ (\autoref{fig:exp2-budget}\refsubfigure{B--C}).
Hence, beyond $S{=}4$, a more accurate likelihood does not compensate for fewer posterior draws, and in practice, we recommend $S{=}2$ for this benchmark.
Under the same budget, \textit{the normalizing flow has a higher C2ST than every mean flow setting} and requires more than four times the wall time of the mean flow ($S{=}2$; \autoref{fig:runtimes}\refsubfigure{B}).
Regular flow matching also has a higher C2ST despite requiring more evaluations, and the same mean flow is about $50\times$ faster.

\paragraph{Delayed acceptance MCMC.}
Screening proposals with a coarse mean flow and accepting them against the fine evaluation of the same network gives an accuracy comparable to NUTS (\autoref{tab:exp2}).
With the coarse acceptance rate adapted to $0.3$, the fine level accepts 80 to 85\% of the evaluated proposals, rising with $S_c$.
Delayed acceptance MCMC is about 26--37 times faster than the flow matching baseline, but slower than NUTS with the mean flow likelihoods and yields a lower ESS.

In summary, a few mean flow steps with more posterior draws give a lower posterior error than both the normalizing flow and the flow matching baseline, at lower wall time.
Delayed acceptance MCMC with the same network gives comparable accuracy without gradients.

\section{DISCUSSION}
\label{sec:discussion}

Neural density estimators for SBI have so far traded flexibility for fast density evaluation. Conditional mean flows relax this trade-off by allowing the step count to be chosen after training. In low-dimensional benchmarks, one or two steps matched the accuracy of flow matching while reducing MCMC time by up to two orders of magnitude. Few steps also yielded accurate posterior samples, and under a fixed MCMC budget, more draws were more valuable than more steps.

\textbf{Limitations.} The main limitation is the cost of density evaluation in high dimensions.
Each step requires its full Jacobian and a log-determinant, whose cost grows cubically with the dimension, and differentiating the log-determinant for NUTS involves second-order derivatives.
Together with the higher cost of each transformer evaluation, this made MCMC in the renewal application slower than with the normalizing flow despite using fewer network evaluations.
In addition, the cycle loss encourages but does not guarantee invertibility, so the learned densities may not exactly be normalized.
Finally, the coverage of every method, including flow matching with 200 steps, decreased at large trial counts, consistent with error accumulation.

\textbf{Outlook.} Beyond MCMC, fast mean flow densities can support self-consistency objectives \citep{schmitt2024leveraging} and, together with prior and likelihood evaluations, evidence estimation \citep{kucharsky2026improving}.
Since these methods also rely on repeated density evaluation, our work could make flexible flow-based estimators practical in any SBI setting that requires densities.

\clearpage

\bibliographystyle{abbrvnat}
\begingroup\raggedright
\let\v\oldv
\let\c\oldc
\bibliography{references}
\endgroup

\clearpage
\appendix
\onecolumn
\renewcommand{\thefigure}{S\arabic{figure}}
\setcounter{figure}{0}
\renewcommand{\thetable}{S\arabic{table}}
\setcounter{table}{0}

\section{Proofs}
\label{app:proofs}

\subsection{Conditional Mean Flows}
\label{app:conditional}

Conditioning reduces the main population objective to an unconditional mean flow objective for each conditional target, as in conditional diffusion and flow matching \citep{vincent2011connection, batzolis2021conditional, arruda2025diffusion}.
We write $\Psi^{\x}_{s\to t}$ for the ODE flow of $\v^*$ and $\u^*$ for its average velocity, defined by $\Psi^{\x}_{s\to t}(\z)=\z+(t-s)\u^*(\z,s,t,\x)$, with diagonal value $\u^*(\z,s,s,\x)=\v^*(\z,s,\x)$.

\paragraph{Assumptions.}
\begin{enumerate}[label=(A\arabic*)]
    \item \emph{Noise.} $\epsilonb\sim\normal(\0,\I)$ is independent of $(\z_0,\x)$, and $\E\|\z_0\|^2<\infty$.
    \item \emph{Time pairs.} $(s,t)$ is independent of $(\z_0,\x,\epsilonb)$, and its distribution has an absolutely continuous component whose density is strictly positive on $(0,1)^2$.
    \item \emph{Regularity.} For $p(\x)$-almost every $\x$, the conditional target $p(\z_0\given\x)$ has a density, and $\v^*(\cdot,\cdot,\x)$ is continuously differentiable in state and time and globally Lipschitz in the state, uniformly over $[0,1]$.
    \item \emph{Realizability.} The admissible pairs $(\u_{\psib},\hat{\v}_{\psib})$ are jointly measurable in $\x$, continuously differentiable in state and time, and have finite population loss, and they contain the exact pair $(\u^*,\v^*)$ simultaneously for all conditions.
\end{enumerate}
(A3) gives a global continuously differentiable flow through both endpoints.
For targets with compact support or discontinuous densities, such as the uniform priors in two moons, $\v^*$ is in general not Lipschitz near $t{=}0$.
For such targets, the argument below applies on $[\delta, 1]$ for every $\delta > 0$ and, for schedules starting at $t_0=\delta$, yields the density of the slightly smoothed target $\z_\delta = (1-\delta)\z_0 + \delta\epsilonb$, which converges to the target as $\delta \to 0$.
In practice, we do not truncate the schedule and keep $t_0 = 0$, since $s = 0$ is almost surely never sampled, and the network's value at $s = 0$ is the continuous extension of the fitted values at small $s$.

{\renewcommand{\theproposition}{\ref{prop:conditional}}\addtocounter{proposition}{-1}
\begin{proposition}[Conditional mean flows]
Assume (A1)--(A4), and let $\loss_{\mathrm{MF}}$ be unweighted ($w{=}1$).
Then every population minimizer satisfies, for $p(\x)$-almost every $\x$,
\begin{equation*}
    \hat{\v}_{\psib}(\z,s,\x)=\v^*(\z,s,\x)
    \quad\text{and}\quad
    \u_{\psib}(\z,s,t,\x)=\u^*(\z,s,t,\x)
    \qquad\text{for all } \z\in\R^d,\ s,t\in[0,1].
\end{equation*}
Consequently, for every schedule $0=t_0<\dots<t_S=1$, the composed map $T$ equals $\Psi^{\x}_{0\to1}$, a diffeomorphism that transports $p(\z_0\given\x)$ to $\normal(\0,\I)$, and \eqref{eq:composed-density} equals $\log p(\z\given\x)$.
\end{proposition}}

\begin{proof}[Proof of Proposition~\ref{prop:conditional}]
Let $\mathbf Y=\epsilonb-\z_0$ and $\mathbf m=\v^*(\z_s,s,\x)$.
By (A1) and (A2), $\E[\mathbf Y\given\z_s,s,t,\x]=\mathbf m$, and (A1) ensures that all terms are finite.
Both predictions in the main loss are functions of these inputs, so the conditional-expectation decomposition gives
\begin{equation*}
    \loss_{\mathrm{MF}}
    =\big(1+\tfrac{1}{d}\big)\E\|\mathbf Y-\mathbf m\|^2
     +\E\|\mathbf V_{\psib}-\mathbf m\|^2
     +\tfrac{1}{d}\E\|\hat{\v}_{\psib}-\mathbf m\|^2.
\end{equation*}
By (A4) and the MeanFlow identity, the exact pair $(\u^*,\v^*)$ makes the last two terms zero.
Every minimizer therefore satisfies $\hat{\v}_{\psib}=\mathbf V_{\psib}=\v^*$ almost everywhere under the training distribution.
For almost every condition, the Gaussian interpolation has full support at interior times.
Together with the time-pair density and continuity, this extends the equalities to all states and times.

Fix such a condition and a trajectory $\z_\tau=\Psi^{\x}_{s\to\tau}(\z)$.
Since $\hat{\v}_{\psib}=\v^*$, the total derivative in $\mathbf V_{\psib}$ is taken along this trajectory, and for fixed $t$,
\begin{equation*}
    \frac{\mathrm d}{\mathrm d\tau}\flowmap(\z_\tau;\tau,t,\x)
    =\v^*(\z_\tau,\tau,\x)-\mathbf V_{\psib}(\z_\tau,\tau,t,\x)=\0.
\end{equation*}
Since $\flowmap(\z_t;t,t,\x)=\z_t$, we obtain $\flowmap=\Psi^{\x}$ and hence $\u_{\psib}=\u^*$, with the diagonal identity following by continuity.
Since $\z_1=\epsilonb$, the flow of $\v^*$ transports $p(\z_0\given\x)$ to $\normal(\0,\I)$ under (A3) \citep{lipman2023flow, liu2022flow}.
Flow composition gives $T=\Psi^{\x}_{0\to1}$ for every schedule, and change of variables gives \eqref{eq:composed-density}.
\end{proof}

The exact flow map satisfies $\Psi^{\x}_{t\to s}\circ\Psi^{\x}_{s\to t}=\mathrm{id}$ for every condition, including swapped ones, so the exact pair also attains the minimum $\loss_{\mathrm{cyc}}=0$.
Hence, adding $\lambda\,\loss_{\mathrm{cyc}}$ does not change the set of population minimizers.
For $p > 0$, the weight $w$ depends on the regression target $\mathbf Y$ through the residual, and the decomposition above no longer applies.
The adaptive weight thus acts as a robust loss that down-weights large residuals \citep{geng2025mean}, and in general changes the learned velocity, but empirically seems to stabilize training (see \autoref{tab:hyperparameters}).

\subsection{Marginalization by Masking}
\label{app:marginal}

We show that the masked objective of \autoref{sec:backbone} targets the marginal distribution of the retained target coordinates.
For a set $\mathcal{O} \subseteq \{1, \dots, d\}$, we write $\z_{\mathcal{O}}$ for the retained coordinates of $\z$, and $\u_{\psib,\mathcal{O}}$ and $\hat{\v}_{\psib,\mathcal{O}}$ for the outputs of the network on $\mathcal{O}$.

{\renewcommand{\theproposition}{\ref{prop:marginal}}\addtocounter{proposition}{-1}
\begin{proposition}[Marginalization by masking]
Let the mask $\mathcal O\subseteq\{1,\dots,d\}$ be nonempty with law $\pi$, independent of $(\z_0,\x,\epsilonb,s,t)$, let the network outputs on $\mathcal O$ depend on the target only through $\z_{\mathcal O}$, and let the unweighted loss be restricted to the coordinates in $\mathcal O$.
Assume (A1)--(A3) for each marginal target $\z_{0,\mathcal O}$ with $\pi(\mathcal O)>0$, and (A4) simultaneously for all such masks.
Then every population minimizer satisfies, for every $\mathcal O$ with $\pi(\mathcal O)>0$ and $p(\x)$-almost every $\x$, $\hat{\v}_{\psib,\mathcal O}=\v^*_{\mathcal O}$ and $\u_{\psib,\mathcal O}=\u^*_{\mathcal O}$, where $\u^*_{\mathcal O}$ is the average velocity of
\begin{equation*}
    \v^*_{\mathcal O}(\z_{\mathcal O},s,\x)
    =\E[\epsilonb_{\mathcal O}-\z_{0,\mathcal O}
          \given\z_{s,\mathcal O}=\z_{\mathcal O},\x].
\end{equation*}
Consequently, \eqref{eq:composed-density}, with the base density and Jacobians evaluated in $|\mathcal O|$ dimensions, equals $\log p(\z_{\mathcal O}\given\x)$ for every schedule.
\end{proposition}}

\begin{proof}[Proof of Proposition~\ref{prop:marginal}]
For a fixed mask, independence leaves the simulation law unchanged, and the coordinate-wise interpolation is
\begin{equation*}
    \z_{s,\mathcal O}=(1-s)\z_{0,\mathcal O}+s\epsilonb_{\mathcal O},
    \qquad\epsilonb_{\mathcal O}\sim\normal(\0,\I_{|\mathcal O|}).
\end{equation*}
The masked main loss is therefore a positively weighted sum of the conditional mean flow objectives for the marginal targets.
By the simultaneous realizability in (A4), the exact pairs attain every summand's lower bound, so every minimizer attains each one.
Proposition~\ref{prop:conditional} applied to $\z_{0,\mathcal O}$ gives the result.
\end{proof}

Additional attention restrictions are covered only when they preserve the ability to represent $\v^*_{\mathcal O}$ and its average velocity.
The independence of the synthetic training mask holds by construction and does not impose a missingness mechanism on the observed data.
The marginal density uses only the observed coordinates.
For $p > 0$, the residual and hence $w$ depend only on the retained coordinates, so the reduction to the marginal targets still holds, but each summand then has the weighted minimizer discussed in Appendix~\ref{app:conditional}.

\section{Evaluation Metrics}
\label{app:metrics}

All experiments use 10 independent datasets, and we report the mean and standard deviation of the following metrics over datasets.

\paragraph{Classifier two-sample test (C2ST).}
The C2ST compares two distributions through a classification task \citep{lopez-paz2017revisiting}.
If the approximate and the reference posterior coincide, no classifier can tell their draws apart better than chance, whereas any difference between them can be exploited by a sufficiently flexible classifier.
Hence, the held-out accuracy $\hat{a}$ of a trained classifier measures how distinguishable the two distributions are: $\hat{a} = 0.5$ indicates indistinguishable and $\hat{a} = 1$ perfectly separable draws.
We label approximate draws with $1$ and reference draws with $0$, subsample both sets without replacement to the same size, which is fixed within each experiment across all methods and step counts, standardize the pooled draws, and train a classifier $c(\thetab) \in (0, 1)$ with the binary cross-entropy.
The classifier is an MLP with two residual hidden layers of width $2^{\lceil \log_2 (10D) \rceil}$ for $D$-dimensional draws, Mish activations, dropout of $0.1$ and a sigmoid output unit.
We train it with AdamW (learning rate $5{\times}10^{-4}$ with linear warmup and cosine decay, weight decay $5{\times}10^{-3}$, gradient clipping at norm $1.5$) and batch size 128 for at most 100 epochs, with early stopping (patience 5) on 20\% of the training data.
The accuracy of the classifier $c(\thetab)$ on held-out draws is averaged over 5-fold cross-validation.

\paragraph{Grid density metrics.}
On two moons, we evaluate the exact posterior $p$ and each learned density $q$ on a $512{\times}512$ midpoint grid over the prior support with cell area $\Delta A$.
We report the total mass $M = \sum_j q_j \Delta A$ of the learned density and compare the shapes of both densities after normalizing them to cell probabilities $\tilde{p}_j$ and $\tilde{q}_j$ on the grid,
\begin{equation*}
    \mathrm{KL} = \sum_j \tilde{p}_j \log \frac{\tilde{p}_j}{\tilde{q}_j}, \qquad
    H = \Big(\tfrac{1}{2} \sum_j \big(\sqrt{\tilde{p}_j} - \sqrt{\tilde{q}_j}\big)^2\Big)^{1/2}.
\end{equation*}
The Hellinger distance $H \in [0, 1]$ remains finite when the learned density vanishes where the exact posterior has mass, whereas the Kullback-Leibler (KL) divergence becomes infinite.

\paragraph{Normalized RMSE and coverage.}
For posterior draws $\hat{\thetab}^{(l)}$, $l = 1, \dots, L$, and the generating parameter $\thetab^*$ of dimension $D$, we compute
\begin{equation*}
    \mathrm{NRMSE} = \bigg(\frac{1}{LD} \sum_{l=1}^{L} \sum_{j=1}^{D} \Big(\frac{\hat{\theta}_j^{(l)} - \theta_j^*}{r_j}\Big)^2\bigg)^{1/2},
\end{equation*}
where $r_j$ is the prior range of parameter $j$.
Coverage is the fraction of parameters whose generating value lies within the central 95\% credible interval.

\paragraph{Posterior contraction.}
We define posterior contraction as the relative reduction in variance from prior to posterior,
\begin{equation*}
    \mathrm{Contraction} = 1 - \frac{1}{D} \sum_{j=1}^{D} \frac{\Var_{\mathrm{post}}(\theta_j)}{\Var_{\mathrm{prior}}(\theta_j)},
\end{equation*}
which measures how much the data reduce the prior uncertainty, but not whether the posterior is correct.

\paragraph{TARP.}
TARP estimates the expected coverage of credible regions around random reference points \citep{lemos2023sampling-based} directly from the trained estimator, without MCMC or reference posteriors.
We report the area between the TARP curve and the diagonal (ATC; 0 is calibrated) and the $p$-value of the Kolmogorov-Smirnov test, using 2{,}000 validation pairs split into 10 folds of 200 pairs with 200 draws each (\autoref{tab:tarp}).
On two moons, TARP assesses the learned posterior; for the other experiments, it assesses whether the learned likelihood generates calibrated data.

\begin{table}[ht]
    \caption{\emph{TARP calibration of every trained estimator.}
    Area between the TARP curve and the diagonal (ATC; 0 is calibrated, positive is overdispersed), mean (standard deviation) over 10 folds of 200 prior-predictive validation pairs (2{,}000 in total) with 200 draws each.
    Two moons calibrates the posterior; g-and-k, the renewal model and the DDM calibrate the likelihood.
    $\dagger$ marks methods whose Kolmogorov-Smirnov $p$-value is below 0.01; the ATC closest to zero per experiment is in bold.}
    \label{tab:tarp}
    \centering
    \small
    \begin{tabular}{lll}
        \toprule
        \textbf{Experiment} & \textbf{Method} & \textbf{ATC} \\
        \midrule
        Two moons & Normalizing flow & $-0.008$ (0.019) \\
        & Mean flow, $S{=}1$ / $S{=}2$ / $S{=}32$ & $-0.001$ (0.010) / $+0.007$ (0.011) / $+0.008$ (0.011) \\
        & Consistency model, $S{=}1$ / $S{=}2$ / $S{=}32$ & $+0.102$ (0.011)$^\dagger$ / $+0.113$ (0.011) / $+0.001$ (0.014) \\
        & Mean flow ($\v$-head), Euler $S{=}1$ / $S{=}2$ / $S{=}32$ & $-0.179$ (0.013)$^\dagger$ / $-0.134$ (0.015) / $-0.007$ (0.012) \\
        & Flow matching, Euler $S{=}1$ / $S{=}32$ & $-0.176$ (0.011)$^\dagger$ / $-0.004$ (0.014) \\
        & Mean flow ($\v$-head) / flow matching, adaptive & $-0.007$ (0.012) / $\mathbf{+0.000}$ (0.011) \\
        \midrule
        g-and-k & Mean flow, $S{=}1$ / $S{=}8$ & $-0.004$ (0.012) / $-0.009$ (0.011) \\
        & Normalizing flow & $\mathbf{+0.001}$ (0.011) \\
        & Flow matching, Euler $S{=}200$ & $+0.004$ (0.007) \\
        \midrule
        Renewal & Mean flow, masked transformer & $\mathbf{-0.025}$ (0.013) \\
        & Normalizing flow, sentinel & $+0.130$ (0.011)$^\dagger$ \\
        \midrule
        DDM & Mean flow, $S{=}1$ / $S{=}8$ & $\mathbf{-0.003}$ (0.009) / $-0.006$ (0.013) \\
        & Flow matching, Euler $S{=}200$ & $\mathbf{-0.003}$ (0.016) \\
        \bottomrule
    \end{tabular}
\end{table}

\paragraph{Posterior predictive checks.}
For the DDM, we simulate a replicate dataset for each of 200 posterior draws and compare the observed choice proportion, the conditional response-time quantiles (10, 30, 50, 70 and 90\%) and the timeout rate $T$ with their replicate distribution.
For each summary, we compute the discrepancy $z = \big(T(\y_{\mathrm{obs}}) - \bar{T}_{\mathrm{rep}}\big)/\mathrm{sd}(T_{\mathrm{rep}})$ and the predictive $p$-value $P\big(T(\y_{\mathrm{rep}}) \ge T(\y_{\mathrm{obs}})\big)$, and report the largest $|z|$ across summaries.

\paragraph{Convergence diagnostics.}
We compute split-$\hat{R}$ and bulk ESS following \citet{vehtari2021rank-normalization}.
A run passes our convergence criteria if $\max \hat{R} \le 1.05$, the minimum bulk ESS is at least 100 and at least 500 draws are kept.
Reported means and standard deviations include runs that fail these criteria.

\section{Architectures, Training and Hyperparameter Search}
\label{app:training}

All networks and diagnostics are implemented in \texttt{BayesFlow} 2.0.14 \citep{radev2023bayesflow, kuehmichel2026bayesflow}, where we added the conditional mean flow as an additional inference network.

\paragraph{Architectures.}
\begin{itemize}
    \item \textbf{Mean flow:} An MLP with 5 hidden layers of width 256.
    The condition $\x$ is embedded by a linear layer and concatenated with the target $\z$ at the input, and, together with a Fourier embedding with scale $0.3$ of the time pair $(s,s-t)$, it modulates every residual block \citep{perez2018film}.
    The auxiliary velocity head shares all but the last 2 layers.
    \item \textbf{Masked transformer:} A transformer \citep{peebles2023scalable} with 5 blocks of width 128 and four attention heads.
    It tokenizes each observation and parameter and receives the time pair $(s, s-t)$ through adaptive layer normalization.
    During training, each observation of each simulation is dropped independently with probability 0.3, redrawn for every batch. Dropped tokens are excluded from attention and loss.
    A causal mask additionally restricts each observation token to parameter and observation tokens up to its own time step.
    The velocity head branches from the last 2 blocks and uses the same masks.
    Each observation and parameter token receives a learned embedding of its coordinate index, which for observations is the time step, and a Fourier embedding is used only for the time pair.
    \item \textbf{Consistency model:} A continuous-time consistency model \citep{lu2025simplifying} with an MLP of 6 hidden layers of width 256.
    \item \textbf{Flow matching:} An MLP with the same depth and width as the mean flow, without the velocity head.
    \item \textbf{Normalizing flow:} A spline coupling flow \citep{durkan2019neural} with 5 coupling layers, each with 2 hidden layers of width 256, and with 2 coupling layers, each with 2 hidden layers of width 128, for the renewal application.
\end{itemize}

\paragraph{Training.}
We trained all networks offline on standardized parameters and observations with a batch size of 256, using AdamW with weight decay $5{\times}10^{-3}$, gradient clipping at norm $1.5$, and a learning rate warmed up linearly over the first 5\% of steps from $5{\times}10^{-5}$ to $5{\times}10^{-4}$ and then decayed to zero with a cosine schedule.
We used 50{,}176 simulations for two moons and the DDM, 100{,}000 for g-and-k, and 8{,}192 for the renewal model to compare architectures under limited data, with 1{,}024 further simulations for validation.
The conditional mean flow, the consistency model, and the flow matching model were trained for 1{,}000 epochs, and the normalizing flow for 200 epochs with early stopping (patience 10).

\paragraph{Hyperparameter search.}
We selected the training hyperparameters of the conditional mean flow with a search in \texttt{Optuna} \citep{akiba2019optuna} on two benchmarks: two moons and a three-dimensional variant of the Gaussian linear task with uniform prior.
\autoref{tab:hyperparameters} lists the search space and the selected values, which all experiments use.
The training times are $s = u_s^{1+\alpha}$ and $t = u_t^{1+\alpha}$ with $u_s, u_t \sim \uniform(0, 1)$, and for $0$ unshared layers the velocity head reads the velocity from the boundary case $\u_{\psib}(\z, s, s, \x)$.
Each trial trained one network per benchmark for 250 epochs of 100 freshly simulated batches of size 256 and evaluated it with $S{=}4$ steps on 256 held-out datasets with 500 posterior draws each.
The score of a benchmark was the sum of an accuracy term, the calibration error, and the grid Hellinger distance to the exact posterior.
The accuracy term is $2\,(\mathrm{C2ST} - 0.5)$ on two moons and the NRMSE on the Gaussian task, and the calibration error is the median absolute difference between the nominal and empirical coverage of central credible intervals at 20 levels.
The C2ST was averaged over 15 and the Hellinger distance over 3 of the held-out datasets.
The search minimized the sum of both benchmark scores with the multivariate tree-structured Parzen estimator of \texttt{Optuna} \citep{akiba2019optuna} on 200 trials.

\begin{table}[ht]
    \caption{\emph{Hyperparameters of the conditional mean flow.}
    Search space of the \texttt{Optuna} search and the selected configuration used in all experiments.}
    \label{tab:hyperparameters}
    \centering
    \small
    \begin{tabular}{lll}
        \toprule
        \textbf{Hyperparameter} & \textbf{Search space} & \textbf{Selected} \\
        \midrule
        Cycle weight $\lambda$  (log) & $[1, 100]$  & $2.5$ \\
        Probability of replacing the condition in $\loss_{\mathrm{cyc}}$ & $[0, 0.8]$ & $0.6$ \\
        Fraction of time pairs with $t{=}s$ & $[0.2, 0.8]$ & $0.7$ \\
        Exponent $p$ of the adaptive weight & $\{0, 0.25, 0.5, 0.75, 1\}$ & $0.25$ \\
        Offset $c$ of the adaptive weight (log) & $[10^{-5}, 1]$ & $0.1$ \\
        Scale of the Fourier time embedding (log) & $[0.02, 30]$ & $0.3$ \\
        Exponent $\alpha$ of the training times & $[-0.5, 4]$ & $0.5$ \\
        Final layers not shared by the velocity head & $\{0, \dots, 5\}$ & $2$ \\
        Adaptive weight on the velocity head & \{yes, no\} & no \\
        \bottomrule
    \end{tabular}
\end{table}

\paragraph{Sampling schedule.}
For two moons, we tuned the sampling schedule of the consistency model and the mean flow separately for each step count, using the C2ST on 2{,}000 draws for 3 held-out observations.
For the consistency model, we selected the discretization exponent from $\rho \in \{1, 2, 3, 3.5, 5, 7, 10\}$.
For the mean flow, we optimized the EDM schedule $(\rho, \sigma_{\min})$ with 7 \texttt{Optuna} trials over $\rho \in [1, 10]$ and $\sigma_{\min} \in [10^{-6}, 10^{-1}]$ (log-uniform), the same number of candidates as for the consistency model.
\autoref{tab:schedule} lists the selected schedules.
For $S{=}1$, the schedule is fixed.
For all other experiments, we set $\rho{=}7$ and $\sigma_{\min}{=}0.001$.

\begin{table}[ht]
    \caption{\emph{Sampling schedules selected on two moons.}
    Discretization exponent $\rho$ of the consistency model and EDM schedule $(\rho, \sigma_{\min})$ of the conditional mean flow for each step count $S$, tuned on held-out observations; at $S{=}1$, the schedule has no free parameter.
    To show how much the selection matters, we report the tuning C2ST (3 held-out observations, 2{,}000 draws each) of the selected and of the worst of the 7 candidates.}
    \label{tab:schedule}
    \centering
    \small
    \begin{tabular}{llccccc}
        \toprule
        & & $S{=}2$ & $S{=}4$ & $S{=}8$ & $S{=}16$ & $S{=}32$ \\
        \midrule
        Consistency model & $\rho$ & 1 & 2 & 3 & 3.5 & 3.5 \\
        & C2ST, selected & 0.918 & 0.542 & 0.555 & 0.539 & 0.535 \\
        & C2ST, worst & 0.973 & 0.975 & 0.946 & 0.700 & 0.666 \\
        \midrule
        Mean flow & $\rho$ & 4.54 & 3.82 & 4.12 & 8.34 & 7.42 \\
        & $\sigma_{\min}$ & $1.1{\times}10^{-4}$ & $4.9{\times}10^{-2}$ & $5.7{\times}10^{-5}$ & $7.0{\times}10^{-2}$ & $2.0{\times}10^{-3}$ \\
        & C2ST, selected & 0.502 & 0.504 & 0.509 & 0.505 & 0.507 \\
        & C2ST, worst & 0.513 & 0.519 & 0.519 & 0.517 & 0.521 \\
        \bottomrule
    \end{tabular}
\end{table}

\paragraph{Code.}
The code and the configurations of all experiments are available under the MIT license at \url{https://anonymous.4open.science/r/MeanFlow-AISTAST}. 
The code is written in Python 3.14, and all networks are \texttt{Keras} 3.15.1 models \citep{chollet2015keras} on the \texttt{JAX} 0.11.1 backend \citep{jax2018github}, both under the Apache 2.0 license.
Besides \texttt{BayesFlow} (MIT), we build on \texttt{PyMC} 5.28.5 (Apache 2.0) \citep{salvatier2016probabilistic} and \texttt{nutpie} 0.16.11 (MIT) \citep{seyboldt2026preconditioning} for MCMC, \texttt{ArviZ} 0.23.4 (Apache 2.0) \citep{martin2026arviZ} for the convergence diagnostics, and \texttt{Optuna} 4.9.0 (MIT) \citep{akiba2019optuna} for the hyperparameter searches.

\paragraph{Compute environment.}
We ran all analyses on a computing cluster using eight CPU cores and one NVIDIA A100 GPU with 40 GB of VRAM.
The computing cluster uses an AMD EPYC 7F72 CPU with a core clock speed of up to 3.2 GHz and 50 GB of RAM.

\section{MCMC Samplers}
\label{app:samplers}

\paragraph{NUTS.}
We sample with the NUTS implementation of \texttt{nutpie} \citep{seyboldt2026preconditioning} through \texttt{PyMC} \citep{salvatier2016probabilistic} on the \texttt{JAX} backend \citep{jax2018github}, with 4 chains, a diagonal mass matrix, a maximum tree depth of 10 and the default target acceptance rate of \texttt{nutpie}.
We use the default of 400 warmup iterations, except for g-and-k, where the warmup is 25\% of the requested draws of each method.

\paragraph{Budget on g-and-k.}
A method whose density costs $K$ network evaluations requests $8{,}000/K$ draws per chain, so that the product of requested draws and network evaluations per density is constant across methods.
The normalizing flow counts as $K{=}1$.
For delayed acceptance MCMC, the requested iterations are $96{,}000/(J S_c + 100)$ per chain, which accounts for the $J{=}5$ coarse evaluations and one fine evaluation per iteration.
The budget is set on the requested draws and ignores the number of leapfrog steps per draw, which differs between methods, so the realized number of network evaluations varies.
We count one network evaluation per forward pass of the network for one pair $(\z, \x)$: a normalizing flow costs one evaluation per density, a mean flow with $S$ steps costs $S$, a velocity field integrated with $S$ Euler steps costs $S$, and a fixed-step Tsit5 integration costs 7 per step.
Jacobians, divergences and gradients are obtained by automatic differentiation on top of these passes and are not counted; their cost is reflected in the wall times.

\paragraph{Delayed acceptance MCMC.}
We implement two-level delayed acceptance MCMC \citep{christen2005markov} with the subchains of \citet{lykkegaard2023multilevel}.
For any coarse posterior $\pi_{\mathrm{c}}$, its invariant distribution is the fine posterior $\pi_{\mathrm{f}}$ (\autoref{sec:sampling}).
Both levels use the same trained network: the coarse level evaluates the mean flow with $S_c \in \{1, 2, 4, 8\}$ steps, and the fine level integrates its velocity head with 100 Tsit5 steps of fixed size and the exact divergence in \eqref{eq:cnf}.
Each iteration runs a subchain of $J{=}5$ random-walk Metropolis steps on $\pi_{\mathrm{c}}$, following \citet{lykkegaard2023multilevel}, and proposes its end point to the fine level with the acceptance probability \eqref{eq:delayed-acceptance}.
The random walk draws Gaussian increments with a step size per parameter, initialized at 10\% of the prior width, and a global scale.
A proposal outside the prior's support is rejected without evaluating the network.
If the subchain does not move, the proposal equals the current state and is accepted without evaluating the fine likelihood.
Adaptation takes 25\% of the requested iterations and is discarded.
During adaptation, a Robbins-Monro update steers the global scale towards a coarse acceptance rate of $0.3$, and in its second half, the step sizes are updated to the marginal standard deviations of the states visited in that half, after which the proposal is fixed.
We run 4 chains started from prior draws.

\section{Per-Experiment Details}
\label{app:tasks}

\subsection{Two Moons}
\label{app:task-two-moons}

The prior is $\thetab \sim \uniform(-1, 1)^2$, and the forward model is
\begin{align*}
    a &\sim \uniform(-\nicefrac{\pi}{2}, \nicefrac{\pi}{2}), \qquad r \sim \normal(0.1, 0.01^2), \\
    \y &= \big(r\cos a + 0.25,\; r \sin a\big) + \big({-}|\theta_1 + \theta_2| / \sqrt{2},\; (-\theta_1 + \theta_2)/\sqrt{2}\big),
\end{align*}
whose fold makes the posterior invariant under $(\theta_1, \theta_2) \mapsto (-\theta_2, -\theta_1)$, which gives two crescent-shaped modes \citep{greenberg2019automatic}.
Since the density is available in closed form and the forward model can be inverted, we draw 32{,}000 exact i.i.d.\ reference samples per observation.
The first dataset uses the canonical observation $\y = (0, 0)$, and the other nine use prior predictive draws.

\subsection{g-and-k}
\label{app:task-gnk}

The g-and-k distribution is defined by its quantile function in normal-score space \citep{rayner2002numerical},
\begin{equation*}
    W(z; g, k) = \big(1 + 0.8 \tanh(g z / 2)\big)\,(1 + z^2)^k\, z, \qquad z \sim \normal(0, 1),
\end{equation*}
and has no closed-form density.
Since $W$ is strictly increasing for $k \ge 0$, we invert it by bisection, which gives the exact log-density by the change of variables.
We couple two coordinates with a Gaussian copula of correlation $\rho$ in normal-score space, such that both marginals are g-and-k distributed.
The priors are $g \sim \uniform(-1, 1)$, $k \sim \uniform(0, 0.5)$ and $\rho \sim \uniform(-0.9, 0.9)$, and each dataset consists of $N{=}50$ observations simulated from a parameter drawn from the prior.
We obtain 32{,}000 reference draws per dataset from an adaptive random-walk Metropolis ensemble on the exact log-posterior, started around the generating parameter with a jitter of 2\% of the prior width (8 chains, 5{,}000 adaptation iterations, 40{,}000 draws thinned by 10, $\hat{R} \le 1.05$ and ESS $\ge 400$).

\subsection{Epidemic Renewal}
\label{app:task-renewal}

On a grid of $T = 128$ steps, of which the first $n_0 = 5$ are seeded at incidence $I_1 = \dots = I_{n_0} = 10$ and not observed, the renewal model is, for $t > n_0$,
\begin{align*}
    \Lambda_t &= \textstyle\sum_{s=1}^{15} w_s\, I_{t-s}, \qquad
    I_t = R_t\, \Lambda_t\, \exp\!\big(\sigma_p \varepsilon_t - \tfrac{1}{2}\sigma_p^2\big), \qquad
    y_t = \log I_t + \sigma_o \eta_t ,
\end{align*}
with $\varepsilon_t, \eta_t \sim \normal(0,1)$, $\sigma_p = \sigma_o = 0.05$, and $w$ the Gamma density with mean 5.0 and standard deviation 2.5 \citep{fraser2007estimating, cori2013new}.
This is the renewal equation of \citet{cori2013new} with a discretized generation interval truncated at 15 steps.
The observation is $(y_{n_0+1}, \dots, y_T)$, so $d = 123$.
The parameter is the whitened coefficient vector $\z \in \R^{128}$ with $\log R = L\z$, where $LL^\top$ is a squared-exponential kernel with lengthscale 2.5, marginal standard deviation 0.4 and a $10^{-6}$ ridge \citep{rasmussen2005gaussian}.
Hence, the prior is exactly $\z \sim \normal(\0, \I)$.
Since $L$ is lower triangular, $z_s$ enters $\log R_t$ only for $t \ge s$, and since the recursion runs forward in time, $y_t$ does not depend on $z_s$ for $s > t$.
The causal attention mask of the masked transformer (\autoref{app:training}) encodes exactly this structure.
Observations are standardized separately at each time step.

We withhold each of the observed steps independently with probability 0.3 and draw a new pattern for every dataset.
Since this mechanism is independent of the process and the parameters, the data are missing completely at random, and the posterior only requires the marginal likelihood $p(\y_{\mathcal{O}} \given \z)$ of the observed entries $\mathcal{O}$ \citep{little2019statistical}.
The masked transformer trains on complete simulations with masks drawn independently per simulation and observation at the same rate and thereby learns this marginal for every pattern.
The normalizing flow instead models vectors with the sentinel $-1$ in the missing entries according to~\citet{wang2024missing}.

\subsection{Cognitive Model}
\label{app:task-ddm}

A trial is a Wiener process with drift $\nu$ that starts at $\beta a$ (with fixed relative starting point $\beta = 0.5$) and terminates at the first hit of $0$ or $a$.
The response time adds a non-decision time $t_0$, and trials time out after \SI{10}{\second}.
The priors are $\nu \sim \uniform(-3, 3)$, $a \sim \uniform(0.8, 2)$ and $t_0 \sim \uniform(0.1, 0.4)$, and we simulate with Euler-Maruyama at $\mathrm{d}t = 10^{-4}$.
The network observes the signed response time $y = (2c - 1)\,t$, with $c = 1$ for the upper boundary, which is bijective with the pair of response time and choice.
Since the non-decision time cannot exceed the fastest observed response time, we sample it below this bound, which tightens as trials are added.
Within each dataset, trial counts are nested, and every setting keeps $4\times1{,}000$ NUTS draws after the default warmup.
For the first trial count $N{=}4$, the chains start at the center of the prior support with the default jitter of \texttt{PyMC}.
The chains for the next trial count start at the posterior median of the preceding trial count, clipped to the current support of $t_0$, with independent jitter per chain.

\section{Additional Figures and Tables}
\label{app:figures}

\begin{figure}[ht]
    \centering
    \includegraphics[width=\textwidth]{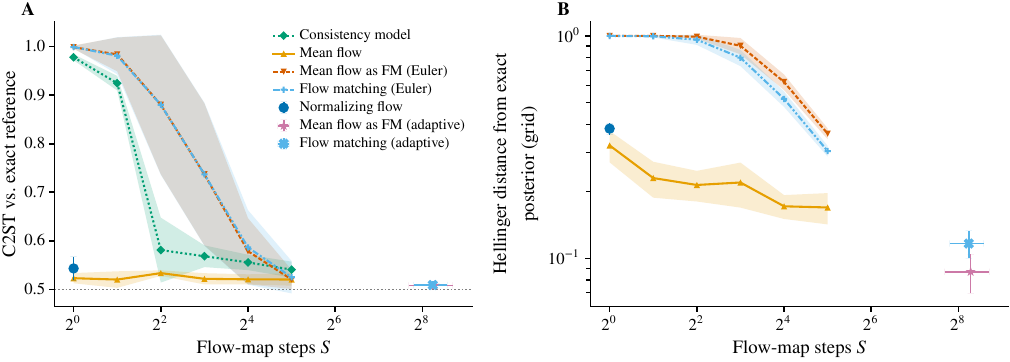}
    \caption{\emph{Sample and density quality against step count on two moons.}
    (\textbf{A})~C2ST against the reference posterior over mean flow steps $S$ ($0.5$ indicates indistinguishable samples).
    (\textbf{B})~Hellinger distance of the learned density from the exact posterior on a $512{\times}512$ grid.
    Lines show the mean over 10 independent datasets, bands the standard deviation.}
    \label{fig:exp1-steps}
\end{figure}

\begin{figure}[ht]
    \centering
    \includegraphics[width=\textwidth,height=0.86\textheight,keepaspectratio]{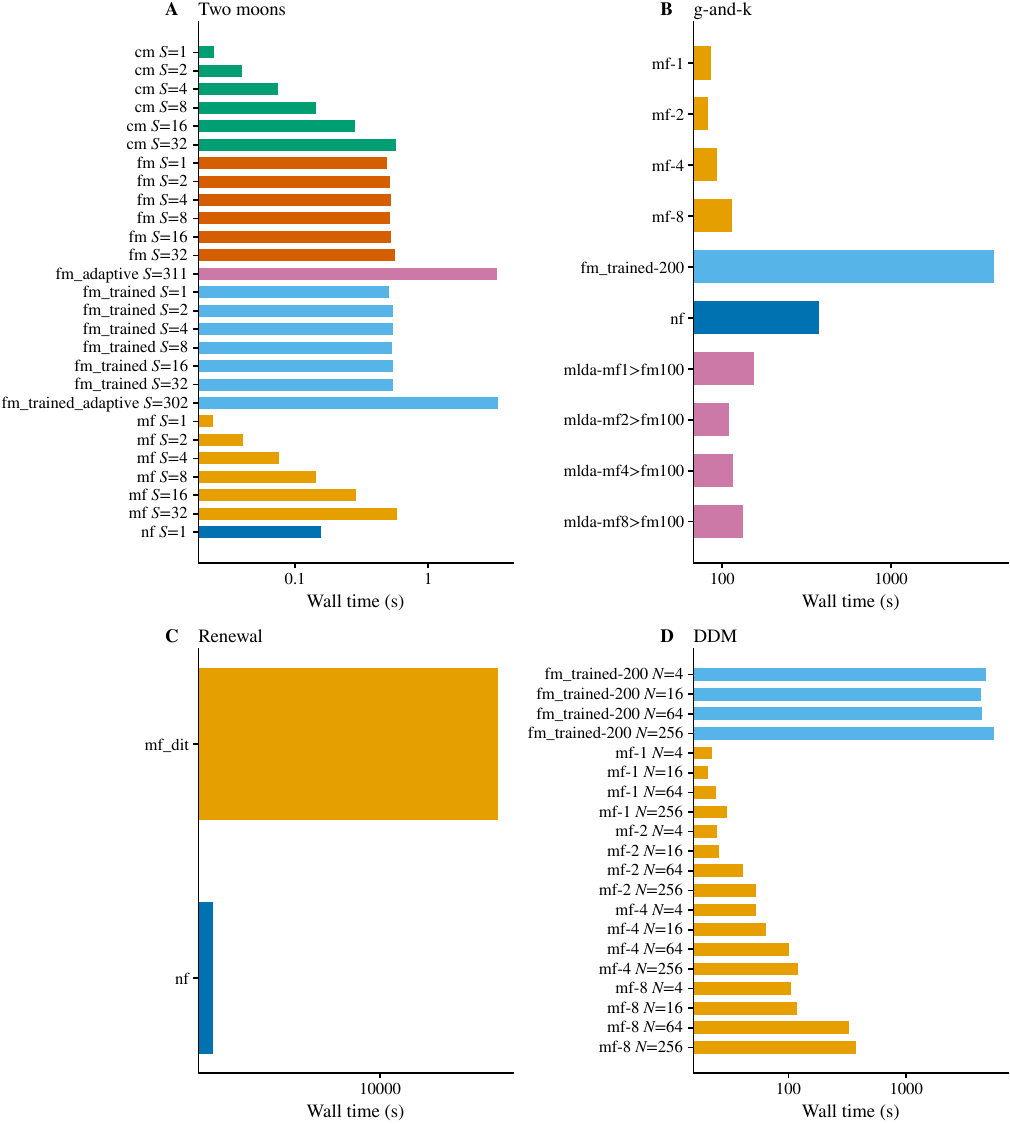}
    \caption{\emph{Wall time of every method across experiments.}
    Warm wall times by method (log scale), mean over 10 datasets, for
    (\textbf{A})~two-moons sampling (5{,}000 draws),
    (\textbf{B})~g-and-k MCMC under the shared budget, with the flow matching baseline outside it,
    (\textbf{C})~renewal MCMC at equal draws and
    (\textbf{D})~DDM MCMC per trial count.}
    \label{fig:runtimes}
\end{figure}

\begin{table}[ht]
    \caption{\emph{Density metrics on two moons.}
    Total mass, KL divergence and Hellinger distance of each learned posterior density against the exact one, on a $512{\times}512$ grid over the prior; mean (standard deviation) over 10 independent datasets.
    The best value per column is in bold, for the mass the one closest to 1.}
    \label{tab:exp1-density}
    \centering
    \small
    \begin{tabular}{lllll}
        \toprule
        \textbf{Method} & \textbf{Steps} & \textbf{Mass} & \textbf{KL} & \textbf{Hellinger} \\
        \midrule
        Normalizing flow & $S{=}1$ & 0.977 (0.055) & 0.446 (0.073) & 0.383 (0.023) \\
        \midrule
        Mean flow & $S{=}1$ & 1.066 (0.017) & 0.347 (0.137) & 0.322 (0.052) \\
        Mean flow & $S{=}2$ & 1.045 (0.030) & 0.179 (0.077) & 0.230 (0.042)  \\
        Mean flow & $S{=}8$ & 1.085 (0.093) & 0.165 (0.053) & 0.219 (0.050)  \\
        Mean flow & $S{=}32$ & $\mathbf{1.003}$ (0.025) & 0.160 (0.052) & 0.169 (0.027) \\
        \midrule
        Mean flow ($\v$-head) & Euler, $S{=}1$ & $7.2{\times}10^{39}$ $(2.3{\times}10^{40})$ & 73.2 (19.0) & 1.000 (0.000)  \\
        Mean flow ($\v$-head) & Euler, $S{=}8$ & 57.3 (45.5) & 4.10 (1.20) & 0.902 (0.076)  \\
        Mean flow ($\v$-head) & Euler, $S{=}32$ & 1.783 (0.077) & 0.385 (0.033) & 0.363 (0.020)  \\
        Mean flow ($\v$-head) & adaptive, $S{=}311$ & 0.993 (0.016) & $\mathbf{0.031}$ (0.013) & $\mathbf{0.087}$ (0.017) \\
        \midrule
        Flow matching & Euler, $S{=}1$ & $5.5{\times}10^{23}$ $(1.7{\times}10^{24})$ & 37.8 (9.5) & 1.000 (0.001)  \\
        Flow matching & Euler, $S{=}32$ & 1.825 (0.073) & 0.271 (0.021) & 0.304 (0.017)  \\
        Flow matching & adaptive, $S{=}302$ & 0.991 (0.023) & 0.048 (0.019) & 0.117 (0.016)  \\
        \bottomrule
    \end{tabular}
\end{table}

\begin{table}[ht]
    \caption{\emph{Composing the DDM trial likelihood, per method and trial count.}
    NRMSE to the generating parameter, coverage of the 95\% central interval, minimum bulk ESS, divergent transitions, and worst predictive discrepancy; mean (standard deviation) over 10 independent datasets.}
    \label{tab:exp4}
    \centering
    \small
    \begin{tabular}{lrllrll}
        \toprule
        \textbf{Method} & $N$ & \textbf{NRMSE} & \textbf{Cov.\ 95\%} & \textbf{Min ESS} & \textbf{Div.} & $\max_k |z_k|$ \\
        \midrule
        Mean flow $S{=}1$ & 4 & 0.268 (0.079) & 0.90 (0.16) & 1{,}396 (312) & 1.0 (1.9) & 0.73 (0.14) \\
        Mean flow $S{=}1$ & 16 & 0.161 (0.033) & 0.90 (0.16) & 1{,}174 (378) & 6.6 (11.5) & 0.90 (0.30) \\
        Mean flow $S{=}1$ & 64 & 0.102 (0.028) & 0.90 (0.16) & 1{,}205 (186) & 2.4 (4.0) & 1.13 (0.47) \\
        Mean flow $S{=}1$ & 256 & 0.070 (0.019) & 0.60 (0.31) & 1{,}146 (335) & 7.8 (21.7) & 2.15 (1.03) \\
        \midrule
        Mean flow $S{=}2$ & 256 & 0.061 (0.016) & 0.73 (0.21) & 1{,}094 (335) & 0.6 (1.6) & 1.93 (0.77)  \\
        Mean flow $S{=}4$ & 256 & 0.066 (0.018) & 0.77 (0.27) & 503 (620) & 50.7 (50.1) & 2.02 (0.77)  \\
        Mean flow $S{=}8$ & 256 & 0.064 (0.023) & 0.70 (0.29) & 513 (504) & 31.8 (29.0) & 1.99 (0.83) \\
        \midrule
        Flow matching, Euler $S{=}200$ & 4 & 0.267 (0.081) & 0.90 (0.16) & 1{,}279 (228) & 0.3 (0.7) & 0.72 (0.13)  \\
        Flow matching, Euler $S{=}200$ & 16 & 0.151 (0.032) & 0.90 (0.16) & 1{,}285 (270) & 0.1 (0.3) & 0.88 (0.22)  \\
        Flow matching, Euler $S{=}200$ & 64 & 0.099 (0.031) & 0.87 (0.23) & 1{,}127 (228) & 0.0 (0.0) & 1.03 (0.57)  \\
        Flow matching, Euler $S{=}200$ & 256 & 0.067 (0.033) & 0.70 (0.33) & 1{,}275 (329) & 0.1 (0.3) & 1.79 (0.57) \\
        \bottomrule
    \end{tabular}
\end{table}

\begin{figure}[ht]
    \centering
    \includegraphics[width=\textwidth]{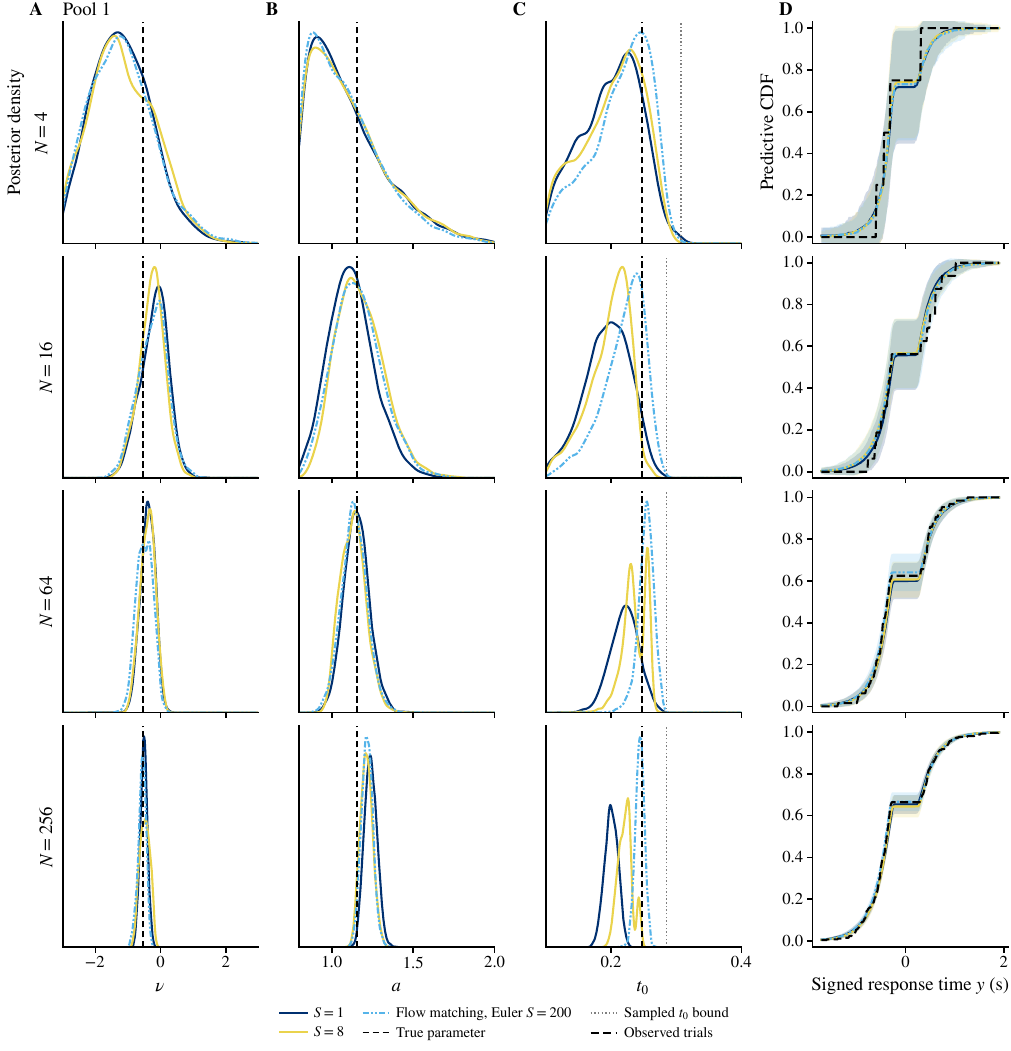}
    \caption{\emph{Composed DDM posteriors on one dataset, over the trial count.}
    Rows are $N \in \{4, 16, 64, 256\}$; columns are the marginals of (\textbf{A})~the drift $\nu$, (\textbf{B})~the boundary separation $a$ and (\textbf{C})~the non-decision time $t_0$ for $S{=}1$, $S{=}8$ and the 200-step baseline, with the generating parameter dashed and the sampled $t_0$ bound dotted.
    (\textbf{D})~Posterior predictive CDF of the signed response time (bands: mean $\pm$ one standard deviation over 200 replicates) against the observed trials.
    Shown for one of 10 datasets.}
    \label{fig:exp4-posteriors}
\end{figure}

\begin{figure}[ht]
    \centering
    \includegraphics[width=\textwidth]{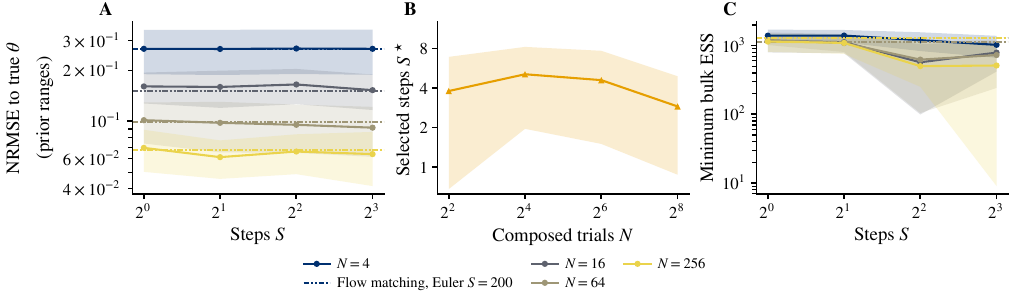}
    \caption{\emph{Effect of step count on DDM inference.}
    \textbf{A}~NRMSE over mean flow steps $S$, one line per trial count $N$, with the flow matching baseline with 200 Euler steps dashed.
    \textbf{B}~Step count selected within each dataset by NRMSE among converged settings, over the trial count $N$.
    \textbf{C}~Minimum bulk ESS over $S$, with the baseline dash-dotted.
    Lines show the mean over 10 independent datasets, bands the standard deviation.}
    \label{fig:exp4-steps}
\end{figure}

\begin{figure}[ht]
    \centering
    \includegraphics[width=\textwidth]{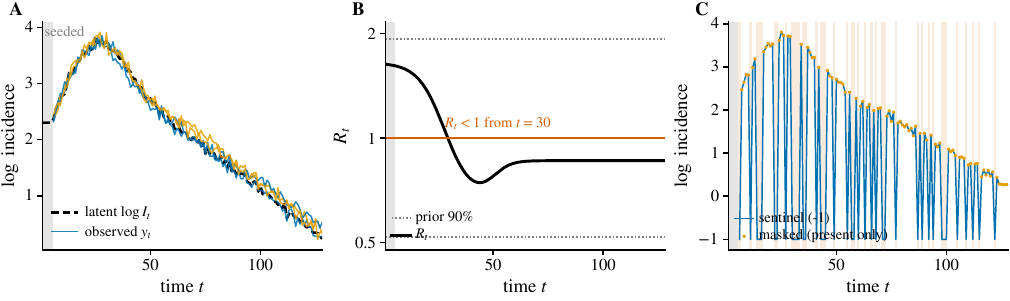}
    \caption{\emph{Renewal-model observations and what each method is shown of them.}
    (\textbf{A})~The observation: four log-incidence curves at one parameter, with the latent $\log I_t$ of the first behind them and the seeded steps shaded, giving 123 observed values on a 128-step grid.
    (\textbf{B})~The reproduction number that generated them, against the $R{=}1$ line and the prior's 90\% interval, for comparison with the generating trajectory.
    (\textbf{C})~The same curve under one 30\% missingness pattern, as each method receives it: the observed coordinates alone for the masked transformer, against the vector with sentinel values for the normalizing flow.}
    \label{fig:exp3-scenario}
\end{figure}


\begin{table}[ht]
    \caption{\emph{Budget-matched MCMC on g-and-k.}
    Every method except the flow matching baseline shares the same budget of requested draws times network evaluations per density. The realized totals in forward passes are $0.9$ to $2.0\times10^{7}$ for the NUTS arms and $1.0$ to $1.4\times10^{8}$ for delayed acceptance MCMC.
    Mean (standard deviation) over 10 independent datasets; the best minimum ESS and C2ST are in bold.}
    \label{tab:exp2}
    \centering
    \small
    \begin{tabular}{lrrrl}
        \toprule
        \textbf{Method} & \textbf{Steps} & \textbf{Samples} & \textbf{Min ESS} & \textbf{C2ST} \\
        \midrule
        Mean flow $S{=}1$ & 1 & 32{,}000 & $\mathbf{10{,}498}$ (1{,}552) & 0.604 (0.049) \\
        Mean flow $S{=}2$ & 2 & 16{,}000 & 4{,}856 (939) & 0.587 (0.046) \\
        Mean flow $S{=}4$ & 4 & 8{,}000 & 2{,}436 (708) & $\mathbf{0.577}$ (0.039) \\
        Mean flow $S{=}8$ & 8 & 4{,}000 & 1{,}143 (476) & 0.599 (0.030) \\
        \midrule
        Normalizing flow & 1 & 32{,}000 & 8{,}023 (3{,}169) & 0.679 (0.078) \\
        Flow matching, Euler $S{=}200$ & 200 & 4{,}000 & 1{,}270 (331) & 0.614 (0.048) \\
        \midrule
        Delayed acceptance MCMC, $S_c{=}1$ & 100 & 3{,}656 & 642 (246) & 0.592 (0.031) \\
        Delayed acceptance MCMC, $S_c{=}2$ & 100 & 3{,}492 & 649 (200) & 0.594 (0.030) \\
        Delayed acceptance MCMC, $S_c{=}4$ & 100 & 3{,}200 & 581 (133) & 0.595 (0.033) \\
        Delayed acceptance MCMC, $S_c{=}8$ & 100 & 2{,}744 & 544 (157) & 0.599 (0.033) \\
        \bottomrule
    \end{tabular}
\end{table}



\end{document}